\documentclass[twocolumn,10pt]{asme2e}
\usepackage{amsmath}
\usepackage{amssymb}
\usepackage{graphicx}
\usepackage{subfigure}
\usepackage{subfigmat}
\usepackage{epstopdf}
\usepackage{hyperref}
\usepackage{algorithmic}
\usepackage{algorithm}
\usepackage{amsfonts}
\usepackage{booktabs} 
\usepackage{array} 
\usepackage{verbatim}
\usepackage{cite}
\newtheorem{lemma}{Lemma}
\newtheorem{theorem}{Theorem}

\title{An Approximation Algorithm for a Shortest Dubins Path Problem}

\author{Sivakumar Rathinam
    \affiliation{
    Associate Professor \\
    Department of Mechanical Engineering\\
     Texas A \& M University\\
    College station, Texas, U.S.A, 77843\\
    }
}

\author{Pramod Khargonekar
    \affiliation{Professor \\
    Department of Electrical and Computer Engineering\\
    University of Florida\\
    Gainesville, Florida, U.S.A., 32611\\
    }
}

\begin{document}
\maketitle

\begin{abstract}
The problem of finding the shortest path for a vehicle visiting a given sequence of target points subject to the motion constraints of the vehicle is an important problem that arises in several monitoring and surveillance applications involving unmanned aerial vehicles. There is currently no algorithm that can find an optimal solution to this problem. Therefore, heuristics that can find approximate solutions with guarantees on the quality of the solutions are useful. The best approximation algorithm currently available for the case when the distance between any two adjacent target points in the sequence is at least equal to twice the minimum radius of the vehicle has a guarantee of $3.04$. This article provides a new approximation algorithm which improves this guarantee to  $ 1+\frac{\pi}{3}\approx 2.04$. The developed algorithm is also implemented for hundreds of typical instances involving at most 30 points to corroborate the performance of the proposed approach. 
\end{abstract}

\section{Introduction}

Advances in sensing, robotics and wireless networks have enabled the use of teams of small Unmanned Vehicles (UVs) for environmental sensing applications including crop monitoring \cite{Geiser_1982, Jackson_1981}, ocean bathymetry \cite{Ferreira}, forest fire monitoring \cite{casbeer}, ecosystem management \cite{pollutant, wildfire}, and civil security applications such as border surveillance \cite{searchandrescue, Krishna2012cdc} and disaster management \cite{disastermanagement}. These applications frequently require vehicles to collect data such as visible/infra-red/thermal images, videos of specified target sites, and environmental data such as temperature, moisture, humidity using onboard sensors, and deliver them to a base station. To accomplish this requirement, small unmanned vehicles with motion and fuel constraints commonly visit a set of target points. There are many fundamental problems that arise here which relate to path planning, control and navigation. This article addresses an important path planning problem that arises in a typical surveillance mission involving a fixed wing Unmanned Aerial Vehicle (UAV) with minimum turning radius constraints.

Given a sequence of target points to visit on a ground plane, this article considers the problem of finding a shortest path passing through the points in the given sequence such that the radius of curvature of any point on the path is at least equal to a positive constant. The curvature constraint imposed on the path models the minimum turning radius of a fixed wing UAV. If the vehicle is traveling at constant speed, this constraint also models the bound on the maximum yaw rate of a fixed wing UAV. Onboard resources such as fuel are limited for small vehicles and therefore, minimizing the length of a vehicle's path can lead to using the resources as efficiently as possible. This problem is referred to as the Curvature constrained Shortest path Problem (CSP) in this article (refer to figure \ref{fig:csp}). The CSP is also one of the two important subproblems of the well known Dubins Traveling Salesman Problem which has received significant attention in the literature\cite{Medeiros2010,Orient2014,macharet2013efficient,macharet2012data,sujit2013route,kenefic2008finding,macharet2011nonholonomic,Ozguner2005,ma2006receding,srtase}.

\begin{figure}[h]
\centering{}
\includegraphics[width=3in]{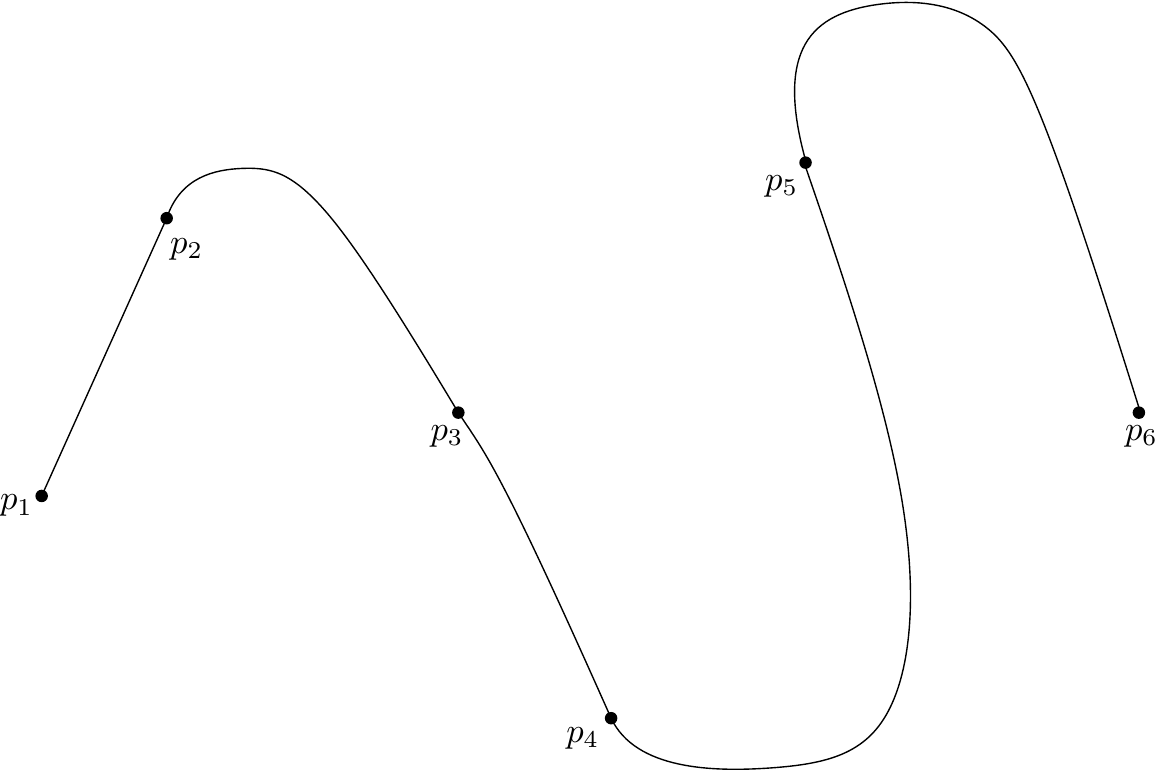}
\caption{A possible feasible path for the CSP visiting a sequence of points denoted by $(p_1,\cdots,p_6)$.}
\label{fig:csp}
\end{figure}

The CSP is a generalization of the two point shortest path problem solved by L.E. Dubins\cite{dubins} in 1957. Dubins\cite{dubins} addressed the problem of finding a shortest path for a vehicle to travel from a point at $(x_1,y_1)$ with heading angle $\theta_1$ to a point at $(x_2,y_2)$ with heading angle $\theta_2$ such that the radius of curvature at any point along the path is at least equal to $\rho \geq 0$. Dubins\cite{dubins} showed that any shortest path to this two point problem must consist of at most three segments where each of the segments must be an arc of radius $\rho$ (denoted by $C$) or a straight line segment (denoted by $S$). Specifically, any shortest path is of type $CSC$ or $CCC$, or a degenerate form of these paths. These paths are generally referred to as the Dubins paths in the literature. If the heading angle at each point is known, then the CSP simply reduces to finding an optimal Dubins path between any two adjacent target points in the given sequence. Therefore, solving the CSP requires us to find an optimal heading angle at each of the target points. This is a non-trival problem because the length of the CSP is a non-linear function of the heading angles at the target points.

There is currently no algorithm that can find an optimal solution to the CSP. Therefore, algorithms that can find feasible solutions with theoretical guarantees on the deviation of the solutions from the optimum are useful. The CSP was first considered by Lee et al. in \cite{LeeDubins} where they provide a 5.03-approximation algorithm. The approximation factor of this algorithm has been recently improved to ($2+\frac{2}{\pi} + \frac{\pi}{2}\approx 4.21$) in \cite{Kim2012}. An $\alpha$-approximation algorithm for a minimization problem is an algorithm that finds a feasible solution whose cost is at most $\alpha$ times the optimal cost for any instance of the problem. The factor $\alpha$ is also referred to as the approximation factor of the algorithm.

This article considers an important case of this problem where the distance between any two adjacent points is at least equal to $2\rho$. In many practical applications, the sensors onboard a vehicle covers a wide swath of the monitoring area, and as a result, the vehicle may not be required to visit points that are close\cite{srtase}.  For this case, a 3.04-approximation algorithm was provided by Rathinam et al. in \cite{srtase}. In this article, we improve on this bound and develop a new approximation algorithm with a factor of $1+\frac{\pi}{3}+\epsilon \approx 2.04 + \epsilon$ for any small constant $\epsilon>0$.

\section{Problem Statement}
Consider a sequence of $n$ points denoted by $(p_1,p_2,\cdots , p_n)$ on a plane. Let the position coordinates of the point $p_i$ be denoted as $(x_i,y_i)$. Without loss of generality, assume $n=3k$ where $k$ is any positive integer. The vehicle is first required to visit $p_1$, then $p_2$, and so on. The minimum turning radius of the vehicle is denoted as $\rho$. Let the Euclidean distance between any two adjacent points in the path visited by the vehicle be at least equal to $2\rho$. The objective of the CSP is to find a path such that the path visits each of the points in the order given by $(p_1,p_2,\cdots ,p_n)$, the radius of curvature at any point along the path is at least equal to $\rho$ and the length of the path is a minimum.

In the next section, we present a $(1+\epsilon)$-approximation algorithm for any given $\epsilon>0$ when $n=3$. We then use this three point algorithm to develop an approximation algorithm for the general case in the subsequent section.

\section{$(1+\epsilon)$-Approximation algorithm for 3 points}

Let the Euclidean distance between points $p_1$ and $p_2$ be denoted as $d_{12}$ and the Euclidean distance between points $p_2$ and $p_3$ be denoted as $d_{23}$. Without loss of generality, assume $p_1$ is at $(x_1,y_1)$, $p_2$ is at the origin and $p_3$ is at $(d_{23},0)$. Also, the case when $x_1\leq 0,y_1=0$ is not considered since the optimal solution in this case is just a straight line from $p_1$ to $p_3$. A curved segment $C$ may either require the vehicle to turn in the clockwise direction (this curved segment with a right turn is denoted as $R$) or turn in the counter clockwise direction (this curved segment with a left turn is denoted as $L$). \\

\begin{figure}[h]
\centering{}
\includegraphics[scale=0.85]{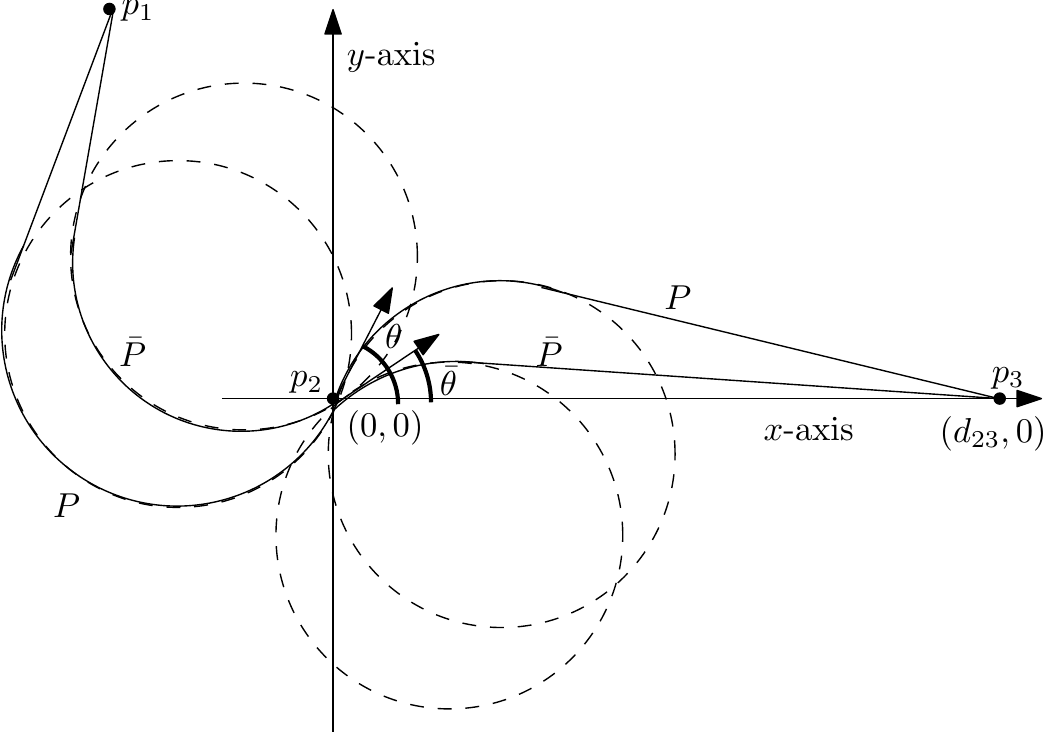}
\caption{Let $P$ denote a $SL$ path from $p_1$ to $p_2$ and a $RS$ path from $p_2$ to $p_3$. The heading angle $\theta$ corresponding to path $P$ at $p_2$ can be reduced by a small amount to say $\bar{\theta}$ to obtain a new path $\bar{P}$ with a shorter length. The difference between the angles $\theta$ and $\bar{\theta}$ is magnified just for illustration.}
\label{fig:proof1}
\end{figure}

\begin{figure}[h]
\centering{}
\includegraphics[scale=0.85]{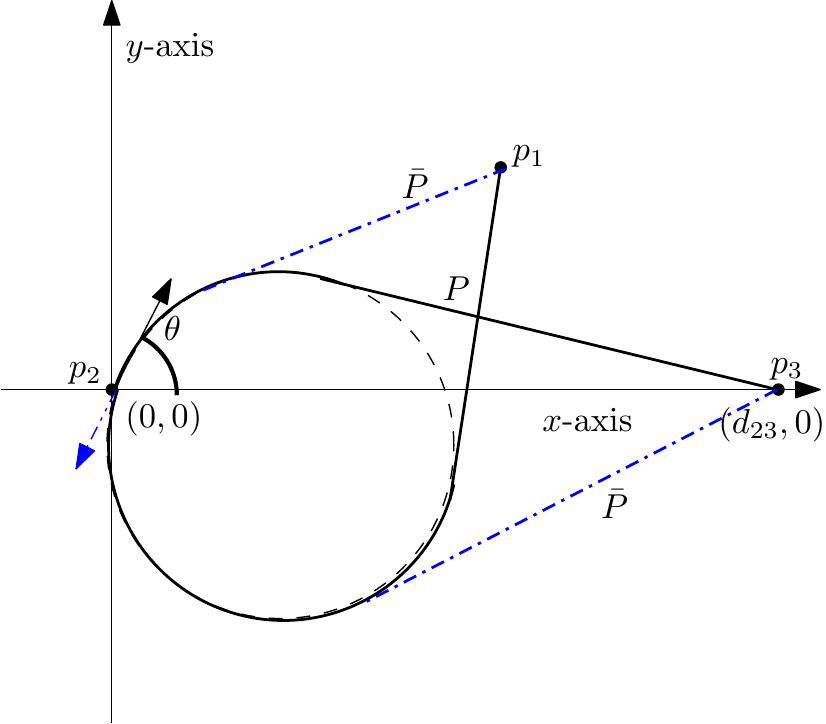}
\caption{In this example $y_1>0$. $P$ here denotes a SRS path from $p_1$ to $p_3$. The straight line segments of $P$ cross each other when $y_1>0$. Therefore, $SRS$ cannot be optimum because one can find a shorter path of type $SLS$ where the straight line segments (dotted blue lines) do not cross each other. It is easy to verify that the length of the $SLS$ path will be less than the length of the $SRS$ path when $y_1>0$.}
\label{fig:proof2}
\end{figure}

\begin{lemma}\label{lemma:3point}
The shortest path of bounded curvature through points $p_1$, $p_2$ and $p_3$ must be of type $SRS$ if $y_1 < 0$, $SLS$ if $y_1>0$ and either $SRS$ or $SLS$ if $y_1=0$. The point $p_2$ lies on the curved segment of the path. \\
\end{lemma}

\begin{proof}
Given the heading angle $\theta$ at $p_2$, as $d_{12}\geq 2\rho$ and $d_{23}\geq 2\rho$, it is known \cite{techreport94} that the shortest path from $p_1$ to $p_2$ must be of type $SC$ while the shortest path from $p_2$ to $p_3$ must be of type $CS$.  For example, the path from $p_1$ to $p_2$ can be of type $SL$ or $SR$. However, it is not optimal to arrive at $p_2$ in the counter clockwise direction and leave $p_2$ in the clockwise direction as shown in figure \ref{fig:proof1}. Similarly, it is also not optimal to arrive at $p_2$ in the clockwise direction and leave $p_2$ in the counter clockwise direction. Therefore, $SCS$ is the only possibility where the vehicle is turning either left or right in the curved segment while visiting $p_2$.  \\

In addition, $SRS$ cannot be optimal when the $y$ coordinate of point $p_1$ is greater than zero (refer to figure \ref{fig:proof2}). Similarly, $SLS$ cannot be optimal when the $y$ coordinate of point $p_1$ is less than zero. Hence proved.  $\blacksquare$
\end{proof}

For the three point problem, if the heading angle at the first point is given, authors in \cite{ma2006receding} prove a similar result by formulating the shortest path problem as an optimal control problem and using Pontryagin's minimum principle. Apart from the difference in our approaches and the fact that we do not have any heading angle constraint at the first point, our work differs from the need to also understand the computational complexity of finding an optimal path. The next theorem shows that given any small $\epsilon>0$, the number of steps required to find an optimum is in the order of $\log_2{\frac{1}{\epsilon}}$ (which is a constant). \\

\begin{theorem}
Given any $\epsilon>0$, a path within $(1+\epsilon)$ times the length of the shortest path can be found in polynomial time for three points.
\end{theorem}

\begin{proof}
From lemma \ref{lemma:3point}, the shortest path is of type $SLS$ or $SRS$. Here, we show that a path of type $SRS$ with length $(1+\epsilon)$ times the optimal length can be found in polynomial time when $y_1<0$. A similar proof can also be provided for $SLS$. Refer to figure \ref{fig:RS} for an illustration of the $SRS$ path when $y_1<0$ and all the notations for the involved angles. Note that the angles $\phi$ and $\psi$ shown in figure \ref{fig:RS} are functions of $\theta$. Let $\mathfrak{D}_1$ be the length of the $SR$ path from $p_1$ to $p_2$ as a function of the heading angle $\theta$ at $p_2$. Let $\mathfrak{D}_2$ be the length of the $RS$ from $p_2$ to $p_3$ as a function of the heading angle $\theta$ at $p_2$. The function $\mathfrak{D}_1(\theta) + \mathfrak{D}_2(\theta)$ is discontinuous at $\theta=0$ and $\theta=\beta$. At all other values of $\theta$, both the length functions $\mathfrak{D}_1(\theta)$ and $\mathfrak{D}_2(\theta)$ are continuously differentiable \cite{Rathinam_lb2015}. Using the results for the derivatives of the length function of an $RS$ path in \cite{Rathinam_lb2015}, wherever the derivatives exist, we have,
\begin{align}
\frac{d\mathfrak{D_1}}{d\theta}  + \frac{d\mathfrak{D_2}}{d\theta} & = -{d_{12}}\sin{\psi}+{d_{23}}\sin{\phi}. \nonumber
\end{align}
$\frac{d\mathfrak{D_1}}{d\theta}  + \frac{d\mathfrak{D_2}}{d\theta} =0$ implies that the sum of the length functions can reach a minimum or a maximum when ${d_{12}}\sin{\psi}={d_{23}}\sin{\phi}$. Now, ${d_{12}}\sin{\psi}$ reaches a maximum value of $2\rho$ when $\theta=\theta_{12}^*$ (refer to figure \ref{fig:RS_3}) and reduces to zero as $\theta$ is further increased to $\beta$. Without loss of generality, assume that $-\pi \leq \theta_{12}^{*} \leq \pi$. It is easy to verify that  ${d_{12}}\sin{\psi}$ is a strictly decreasing function when $\theta \in [\theta_{12}^*,\beta]$. This is because $0<\psi < \frac{\pi}{2}$ and ${ d_{12}\cos{\psi} >L_{12}}$ for $\theta \in (\theta_{12}^*,\beta)$. Therefore, using the results in \cite{Rathinam_lb2015}, we get,
\begin{align}
\frac{d ({d_{12}}\sin{\psi}) }{d\theta} & = -{d_{12}}\cos{\psi} (\frac{ d_{12}\cos{\psi} -L_{12}}{L_{12}}) < 0. \nonumber
\end{align}

Similarly, ${d_{23}}\sin{\phi}$ reaches a minimum value of 0 when $\theta=0$ and increases to $2\rho$ when $\theta$ becomes equal to $\theta_{23}^*$ (refer to figure \ref{fig:RS_2}). Also, ${d_{23}}\sin{\phi}$ is a strictly increasing function when $\theta \in [0,\theta_{23}^*]$.

Now, the set $[\theta_{12}^*,\beta] \cap  [0,\theta_{23}^*]$ is always nonempty because $0 \leq \beta$ and $\theta_{12}^* \leq \frac{\pi}{2} \leq \theta_{23}^*\leq \pi$. Therefore, one can verify that for $\theta \in [\theta_{12}^*,\beta] \cap  [0,\theta_{23}^*]$, the functions ${d_{12}}\sin{\psi}$ and ${d_{23}}\sin{\phi}$ must intersect at some angle $\theta=\theta^*$. In addition, for any $\theta \in (\theta_{12}^*,\beta) \cap  (0,\theta_{23}^*)$,

\begin{align*}
\frac{d^2\mathfrak{D_1}}{d\theta^2}  + \frac{d^2\mathfrak{D_2}}{d\theta^2}  & = {d_{12}}\cos{\psi} (\frac{ d_{12}\cos{\psi} -L_{12}}{L_{12}}) ~ + \\ & \quad \quad  {d_{23}}\cos{\phi} (\frac{ d_{23}\cos{\phi} -L_{23}}{L_{23}}) \\ & >0.
\end{align*}

Hence, $\mathfrak{D}_1(\theta)+\mathfrak{D}_2(\theta)$ reaches an unique minimum at $\theta=\theta^*$. This minimum can be found using an interval bisection algorithm with the number of iterations of the algorithm in the order of $\log_2(\frac{1}{\epsilon})$ for any given $\epsilon>0$ \cite{Nemirovski}. $\blacksquare$
\end{proof}

\begin{figure}[h]
\centering{}
\includegraphics[scale=0.9]{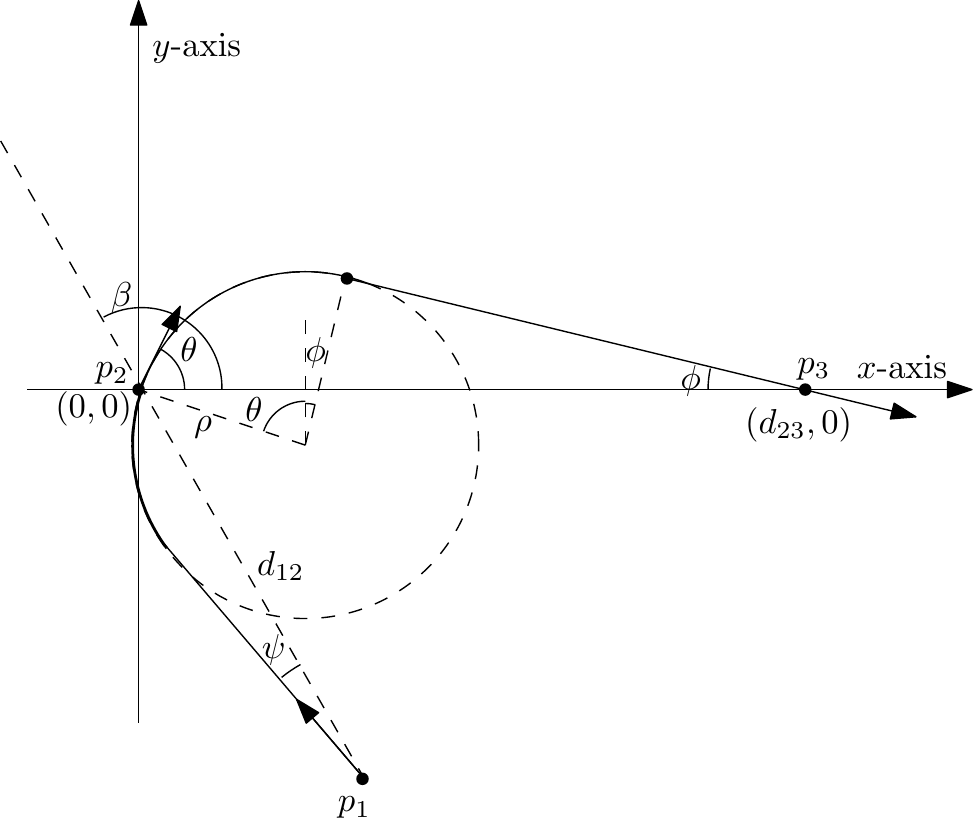}
\caption{A $SRS$ path from $p_1$ to $p_3$ via $p_2$.}
\label{fig:RS}
\end{figure}

\begin{figure}[h]
\centering{}
\includegraphics[scale=0.9]{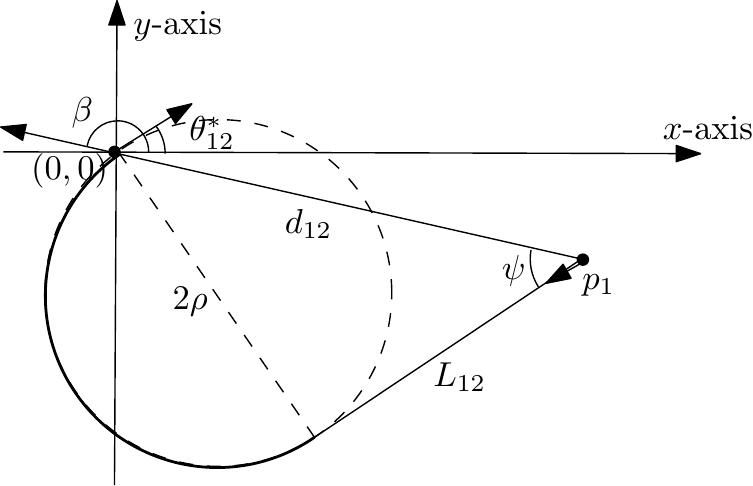}
\caption{$SR$ path from $p_1$ to $p_2$. Note that when $\theta=\theta_{12}^*$, $d_{12}\sin{\psi}$ reaches a maximum value of $2\rho$. As $\theta$ is increased further to $\beta$, $d_{12}\sin{\psi}$ reduces to zero.}
\label{fig:RS_3}
\end{figure}

\begin{figure}[h]
\centering{}
\includegraphics[scale=0.9]{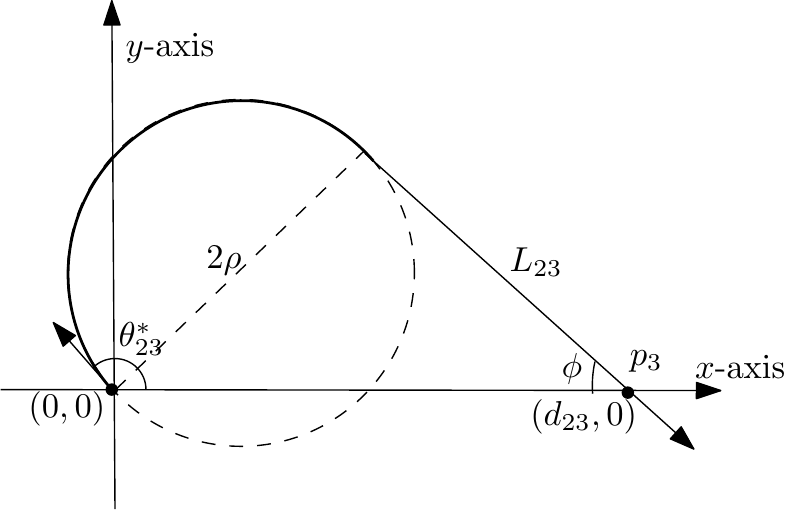}
\caption{$RS$ path from $p_2$ to $p_3$. Note that when $\theta=0$, $d_{12}\sin{\phi}=0$. As $\theta$ is increased further to $\theta_{23}^*$, $d_{12}\sin{\psi}$ increases and reaches its maximum value which is equal to $2\rho$.}
\label{fig:RS_2}
\end{figure}

The following lemma proves an additional property about the nature of the optimal paths for the 3 point problem. This property is useful for numerical implementations and can be used as a termination criteria. Again, the following result is shown for the $SRS$ path. $SLS$ path can be handled in a similar way. \\

\begin{lemma}
For the optimal SRS path when $y_1\leq 0$, the angle of turn in the curved segment from point $p_1$ to $p_2$ is equal to the angle of turn in the curved segment from $p_2$ to $p_3$. \\
\end{lemma}

\begin{proof}
Refer to figure \ref{fig:RS}. Using the result in \cite{Rathinam_lb2015}, we get,
\begin{align}
\frac{d\mathfrak{D_2}}{d\theta} & = {d_{23}}\sin{\phi} \nonumber
& = \rho - \rho\cos{(\theta+\phi)} \nonumber
& = \rho - \rho \cos(turn_{23}) \nonumber
\end{align}
where $turn_{23}= \theta+\phi$ is the turn angle of the curved segment from $p_2$ to $p_3$.
Similarly,
\begin{align}
\frac{d\mathfrak{D_1}}{d\theta} = -{d_{12}}\sin{\psi} \nonumber
& = -(\rho - \rho\cos{(turn_{12}))} \nonumber
\end{align}
where $turn_{12}$ is the turn angle of the curved segment from $p_1$ to $p_2$. \\

Therefore, $\frac{d\mathfrak{D_1}}{d\theta} + \frac{d\mathfrak{D_2}}{d\theta}=0$ implies that $\cos(turn_{12}) = \cos{(turn_{23})}$. Either $turn_{12}=turn_{13}$ or $turn_{12}+turn_{13}=2\pi$. But $turn_{12}+turn_{13}$ cannot be equal to $2\pi$ because this would imply that $y_1 > 0$ which is not true. Therefore, $turn_{12}=turn_{13}$. $\blacksquare$
\end{proof}

\section{Approximation algorithm for $n$ points}

The approximation algorithm first finds three feasible solutions for the CSP and then chooses the best of these three solutions. The three solutions are constructed in the following way:
\begin{enumerate}
\item {\bf Solution \boldmath${\mathcal{F}_1}$}: Use the three point algorithm to find a path for each of the following sequences of points: $(p_1,p_2,p_3), (p_4,p_5,p_6), \cdots, (p_{3k-2},p_{3k-1},p_{3k})$. These paths will fix the heading angle of the vehicle at each of the points. Now, use these heading angles to find the shortest Dubins path from $p_3$ to $p_4$, $p_6$ to $p_7$ and so on. Concatenate all these Dubins paths along with the paths obtained using the three point algorithm such that the resulting solution ($\mathcal{F}_1$) visits the points in the given sequence. \\
\item {\bf Solution \boldmath${\mathcal{F}_2}$}: Use the three point algorithm to find a path for each of the following sequences of points: $(p_2,p_3,p_4), (p_5,p_6,p_7), \cdots, (p_{3k-4},p_{3k-3},p_{3k-2})$. In addition, join points $p_{3k-1}$ and $p_{3k}$ using a line segment. These paths will fix the heading angle of the vehicle at each of the points except $p_1$. Choose any arbitrary angle for visiting $p_1$. Now, use these heading angles to find the shortest Dubins path from $p_1$ to $p_2$, $p_4$ to $p_5$ and so on. Concatenate all these Dubins paths along with the paths obtained using the three point algorithm and the line segment such that the resulting solution ($\mathcal{F}_2$) visits the points in the given sequence. \\
\item {\bf Solution \boldmath${\mathcal{F}_3}$}: Use the three point algorithm to find a path for each of the following sequences of points: $(p_3,p_4,p_5), (p_6,p_7,p_8), \cdots, (p_{3k-3},p_{3k-2},p_{3k-1})$. In addition, join points $p_{1}$ and $p_{2}$ using a line segment. These paths will fix the heading angle of the vehicle at each of the points except $p_{3k}$. Choose any arbitrary angle for visiting $p_{3k}$. Now, use these heading angles to find the shortest Dubins path from $p_2$ to $p_3$, $p_5$ to $p_6$ and so on. Concatenate all these Dubins paths along with the paths obtained using the three point algorithm such that the resulting solution ($\mathcal{F}_3$) visits the points in the given sequence.
\end{enumerate}

Among these three solutions, the approximation algorithm chooses a solution ({\boldmath$\mathcal{F}_a$}) that corresponds to the least of the lengths of $\mathcal{F}_1$, $\mathcal{F}_2$ and $\mathcal{F}_3$. Let $cost(\mathcal{F})$ denote the length of any solution $\mathcal{F}$. The algorithm runs in polynomial time as it primarily relies on solving the three point problems which can be solved in polynomial time for given $\epsilon>0$. The following theorem shows the approximation factor of the algorithm: \\

\begin{theorem}\label{}
Let $\mathcal{F}_{opt}$ and {$\mathcal{F}_a$} respectively denote an optimal solution to the CSP and the solution obtained by the approximation algorithm. Consider any constant $\epsilon>0$. The length of the solution $\mathcal{F}_a$ is at most equal to $(1+\frac{\pi}{3} + \epsilon)$ times the length of $\mathcal{F}_{opt}$. That is, $cost(\mathcal{F}_a) \leq (1+\frac{\pi}{3} + \epsilon)cost(\mathcal{F}_{opt})$. \\
\end{theorem}

\begin{proof}
For $i,j=1,\cdots,n$, $i<j$, let $\mathcal{F}(p_{i},p_{j})$ denote the part of the solution $\mathcal{F}$ from point $p_i$ to point $p_j$. Hence, $cost(\mathcal{F}(p_{i},p_{j}))$ essentially denotes the length of the part of the solution $\mathcal{F}$ from $p_i$ to $p_j$. Also, let $d_{i,j}$ represent the Euclidean distance between points $p_i$ and $p_j$.

\begin{align}\label{bound}
Cost(\mathcal{F}_1) & = \sum_{i=1}^k cost(\mathcal{F}_1(p_{3i-2},p_{3i})) + \sum_{i=1}^{k-1}cost(\mathcal{F}_1(p_{3i},p_{3i+1})).
\end{align}

By construction, for any $i=1,\cdots,k$, the three point algorithm finds a path of bounded curvature with an approximation guarantee of $(1+\epsilon)$ from $p_{3i-2}$ to $p_{3i}$. Therefore, the length of this path must be at most equal to $(1+\epsilon)$ times length of the part of the path $\mathcal{F}_{opt}$ from $p_{3i-2}$ to $p_{3i}$. Hence,
\begin{equation}\label{bound1}
cost(\mathcal{F}_1(p_{3i-2},p_{3i})) \leq (1+\epsilon) cost(\mathcal{F}_{opt}(p_{3i-2},p_{3i})).
\end{equation}
Similarly, a shortest Dubins path is constructed from $p_{3i}$ to $p_{3i+1}$ for all $i=1,\cdots,k-1$. Using the bound shown in \cite{Kim2012}, we get,

\begin{align}\label{bound2}
cost(\mathcal{F}_1(p_{3i},p_{3i+1})) & \leq (1+\pi) d_{3i,3i+1} \nonumber \\
& \leq (1+\pi)cost(\mathcal{F}_{opt}(p_{3i},p_{3i+1})).
\end{align}

Substituting for the bounds from equations \eqref{bound1} and \eqref{bound2} in \eqref{bound}, we get,
{\small
\begin{align}
& cost(\mathcal{F}_1) \nonumber \\
& \leq (1+\epsilon)\sum_{i=1}^k cost(\mathcal{F}_{opt}(p_{3i-2},p_{3i})) + (1+\pi)\sum_{i=1}^{k-1}cost(\mathcal{F}_{opt}(p_{3i},p_{3i+1})) \nonumber \\
& =\underbrace{(1+\epsilon) \sum_{i=1}^k cost(\mathcal{F}_{opt}(p_{3i-2},p_{3i})) + \sum_{i=1}^{k-1}cost(\mathcal{F}_{opt}(p_{3i},p_{3i+1}))}_{\leq (1+\epsilon) cost(\mathcal{F}_{opt}) }  \nonumber \\ & \quad + \pi \sum_{i=1}^{k-1}cost(\mathcal{F}_{opt}(p_{3i},p_{3i+1})) \nonumber \\
& = (1+\epsilon) cost(\mathcal{F}_{opt}) + \pi \sum_{i=1}^{k-1}cost(\mathcal{F}_{opt}(p_{3i},p_{3i+1})). \label{eq:f1}
\end{align}
}

Similarly, we can also bound the length of solutions $\mathcal{F}_2$ and $\mathcal{F}_3$ as follows:
\begin{align}\label{eq:f2}
& cost(\mathcal{F}_2) \leq (1+\epsilon)cost(\mathcal{F}_{opt}) + \pi \sum_{i=1}^{k}cost(\mathcal{F}_{opt}(p_{3i-2},p_{3i-1})).
\end{align}
\begin{align}\label{eq:f3}
& cost(\mathcal{F}_3) \leq (1+\epsilon)cost(\mathcal{F}_{opt}) + \pi \sum_{i=1}^{k}cost(\mathcal{F}_{opt}(p_{3i-1},p_{3i})).
\end{align}

We use the following notations to simplify the terms that appear in equations \eqref{eq:f1},\eqref{eq:f2} and \eqref{eq:f3}. Let,

\begin{align*}
L_1= \sum_{i=1}^{k-1}cost(\mathcal{F}_{opt}(p_{3i},p_{3i+1})). \\
L_2 = \sum_{i=1}^{k}cost(\mathcal{F}_{opt}(p_{3i-2},p_{3i-1})). \\
L_3 = \sum_{i=1}^{k}cost(\mathcal{F}_{opt}(p_{3i-1},p_{3i})).
\end{align*}

One can verify that $L_1 + L_2 + L_3 = cost(\mathcal{F}_{opt})$. Therefore, as $L_1,L_2,L_3\geq 0$, we obtain,
\begin{align}
\min{(L_1,L_2,L_3)}\leq \frac{1}{3}cost(\mathcal{F}_{opt}).
\end{align}

Combining all the bounds for the length of the three solutions and the above equation, we get,
\begin{align*}
&cost(\mathcal{F}_{a})  = \min{(cost(\mathcal{F}_1),cost(\mathcal{F}_2),cost(\mathcal{F}_3)) } \nonumber \\
& \leq (1+\epsilon)cost(\mathcal{F}_{opt})  + \pi \min{(L_1,L_2,L_3)} \\
& \leq (1+\frac{\pi}{3} +\epsilon) cost(\mathcal{F}_{opt}). 
\end{align*}
$\blacksquare$
\end{proof}

\section{Numerical Results}

The approximation algorithm was implemented on problem instances with 12, 15, 18, 21, 24, 27 and 30 target points. For a given number of target points, we generated 100 instances. The turning radius of the vehicle was set to $100$ units. The coordinates of the points were uniformly randomly generated on a 2D plane with the requirement that the minimum distance between any two adjacent points in the sequence is at least twice the turning radius of the vehicle. \\

For a given problem instance $I$, the bound on the {\it a posterior guarantee} provided by the approximation algorithm is defined as $\frac{C^I_{approx}}{C^I_{lb}}$ where $C^I_{approx}$ is the cost of the feasible solution found by the approximation algorithm and $C^I_{lb}$ is the lower bound on the length of the shortest path for the CSP. The parameter $\epsilon$ in the approximation algorithm was set to $10^{-4}$. The lower bound for each instance was obtained using the procedure outlined in \cite{Rathinam_lb2015}. This lower bounding algorithm involved discretizing the set of possible heading angles $[0,2\pi]$ at each point into 32 uniform intervals (refer to \cite{Rathinam_lb2015} for more details on this computation). All the algorithms were coded in MATLAB and the computations were run on a Dell Precision Workstation (Intel Xeon Processor @2.53 GHz, 12 GB RAM). The average running time of the approximation algorithm was in the order of a second for all the instances. \\

The max and average {\it a posterior guarantee} of the solutions found by the approximation algorithm is shown in table \ref{tab:1} along with the theoretical approximation guarantee. Even though the theoretical guarantee is 2.04, the numerical results show that the max and average {\it a posterior guarantee}  was less than 1.28 and 1.15 respectively for the considered instances. These results imply that the proposed algorithm produced solutions with bounds that are significantly better than the guarantees indicated by the approximation factor. The solutions obtained by the approximation algorithm for an instance is shown in figure \ref{fig:path}.

\begin{table}[h!]
  \scriptsize
\caption{Comparison of the theoretical and numerical results for the approximation algorithm}
\vspace{.1cm}
\label{tab:1}
\begin{center}
\begin{tabular}{cccc}
\hline \noalign{\smallskip}
{No. of Points} & \parbox[c][0.65cm]{2cm}{\centering Theoretical upper bound}  & \parbox[c][0.65cm]{2cm}{\centering Max a posteriori bound}  & \parbox[c][0.65cm]{2cm}{\centering Average a posteriori bound} \\
  \hline \noalign{\smallskip}
   12 & 2.04 & 1.27  & 1.09   \\
   15 & 2.04 & 1.25 &1.09   \\
   18 & 2.04 & 1.24 & 1.11  \\
   21 & 2.04 & 1.22 & 1.11  \\
   24 & 2.04 & 1.24 & 1.11 \\
   27 & 2.04 & 1.23 & 1.13 \\
   30 & 2.04 & 1.20 & 1.12 \\
   \hline \noalign{\smallskip}
   \end{tabular}
\end{center}
\end{table}

\begin{figure}
  \centering
  \subfigure[\boldmath${\mathcal{F}_1}$]{\includegraphics[scale=.6]{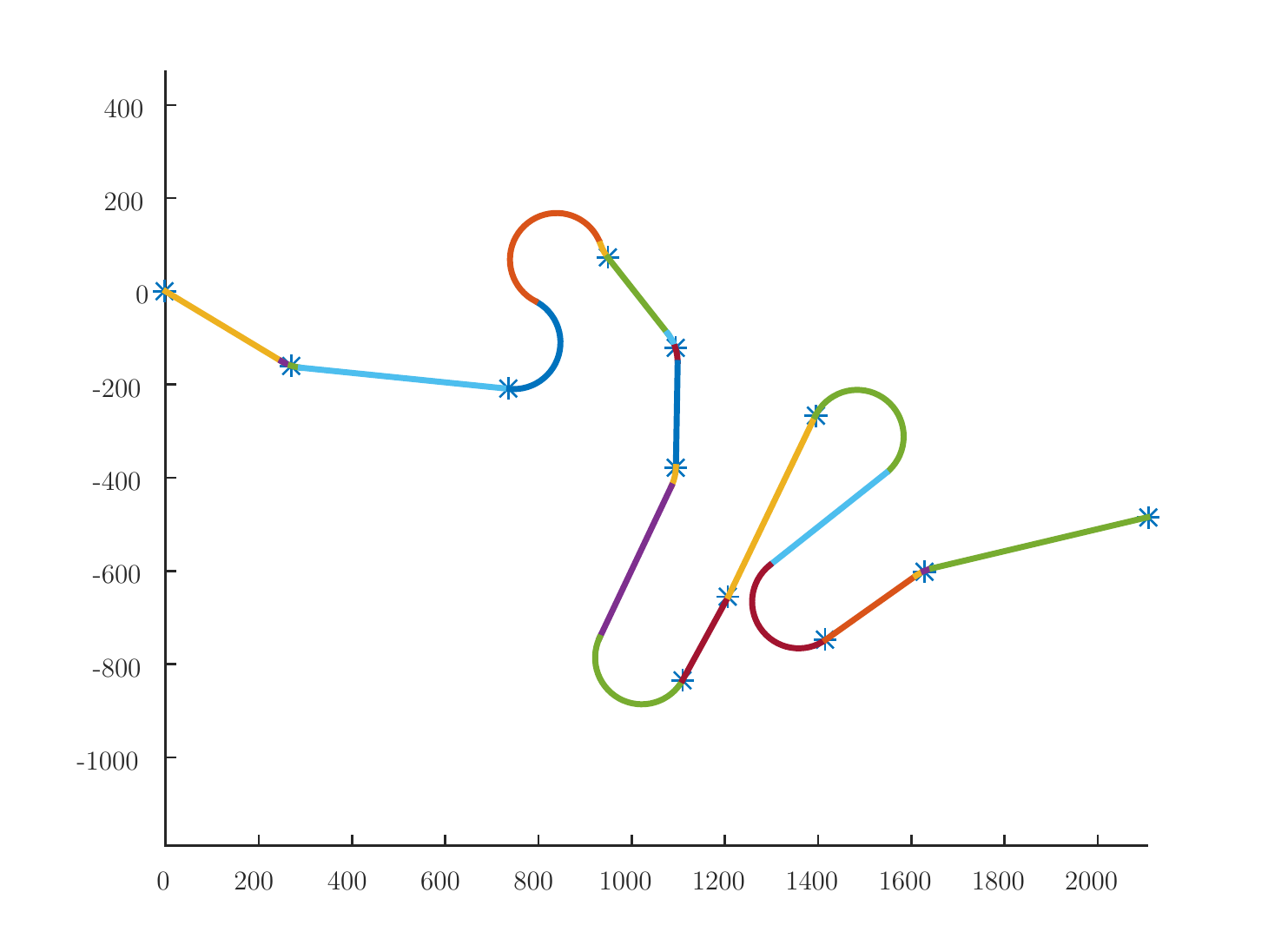}
\label{approxv:step1}}
  \subfigure[\boldmath${\mathcal{F}_2}$]{\includegraphics[scale=.6]{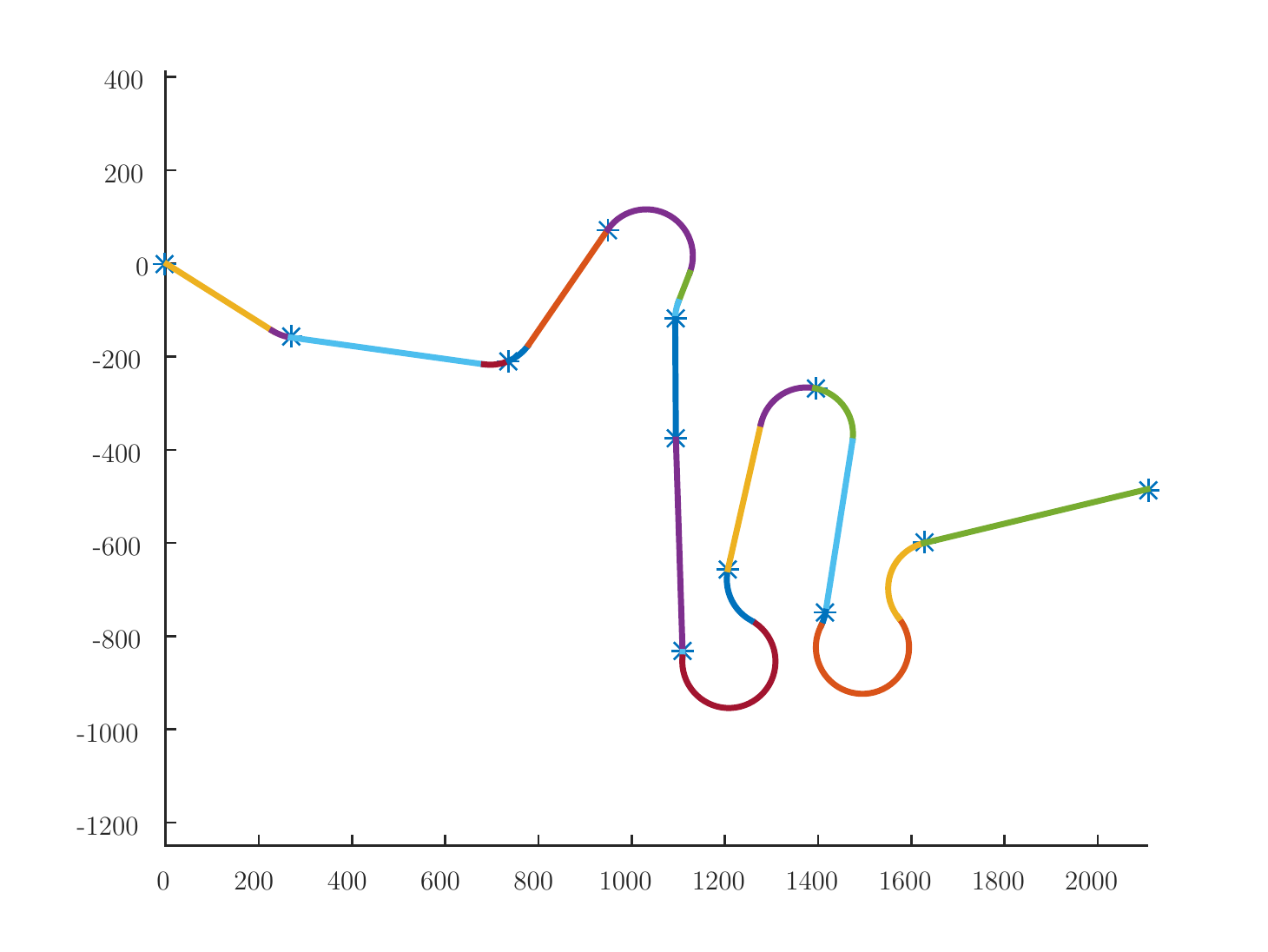}
\label{approxv:step1}}
  \subfigure[\boldmath${\mathcal{F}_3}$]{\includegraphics[scale=.6]{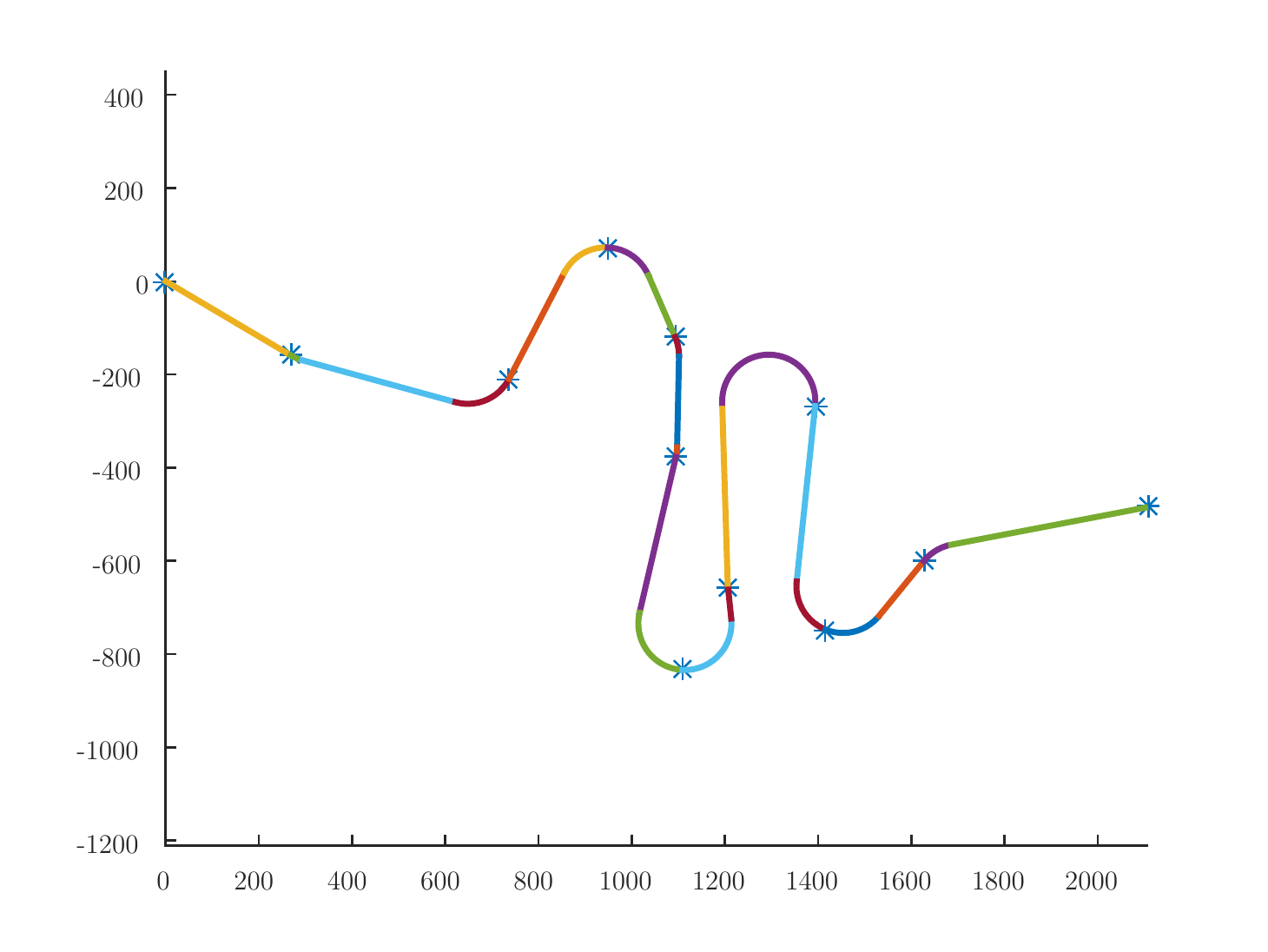}
\label{approxv:step1}}
\caption{The three solutions obtained by the algorithm for a problem instance with 12 points.}
\label{fig:path}
\end{figure}

\section{Conclusions}
This article provided a new approximation algorithm for a curvature constrained shortest path problem (CSP) for visiting a given sequence of target points. First, an algorithm that finds a good path for 3 points is presented. Next, the given sequence is broken into subsequences with 3 points in each subsequence, and a solution is obtained for each subsequence using the three point algorithm. Finally, the solutions corresponding to all the subsequences are concatenated in a suitable way to find a feasible solution to the SPP. This basic idea can be improved in several directions to provide better approximation guarantees for the CSP. First, if one can solve a four point problem with an approximation factor of $(1+\epsilon)$ for some small $\epsilon>0$, using the ideas presented in this article, a guarantee of $(1+\frac{\pi}{4}+\epsilon)$ can be obtained for the CSP. Second, the bounds presented in this article are mainly based on the Euclidean distances with no consideration given to the angle of turn of the vehicle as it passes the points. Including the turn angles has a potential to reduce the approximation factor further. In fact, the algorithm provided in \cite{LeeDubins} uses the constraints on the turn angles in a clever way to provide a constant factor approximation guarantee. We believe that a combination of the ideas presented in this article with lower bounds that are computed based on the Euclidean distances and the turn angles can improve the approximation guarantees significantly even for the more general case when adjacent points are closed spaced.

\bibliographystyle{asmems4}
\bibliography{dubseqbib}

\end{document}